\documentclass[10pt,twocolumn,letterpaper]{article}

\usepackage[pagenumbers]{cvpr} % To force page numbers, e.g. for an arXiv version

\usepackage{booktabs}
\usepackage{dashrule}
\usepackage{array}
\usepackage{arydshln}
\usepackage{amsthm}
\newtheorem{proposition}{Proposition}

\definecolor{cvprblue}{rgb}{0.21,0.49,0.74}
\usepackage[pagebackref,breaklinks,colorlinks,allcolors=cvprblue]{hyperref}
\usepackage{array}
\usepackage{multirow}
\newcolumntype{V}{>{\ttfamily\small}l}

\def\paperID{*****} % *** Enter the Paper ID here
\def\confName{CVPR}
\def\confYear{2026}

\title{Predictive Spectral Calibration for Source-Free Test-Time Regression}

\author{Nguyen Viet Tuan Kiet\\
\small Hanoi University of Science and Technology\\
{\tt\small kiet.nvt220032@hust.edu.vn}
\and
Huynh Thanh Trung\\
\small VinUniversity\\
{\tt\small trung.ht@vinuni.edu.vn}
\and
Pham Huy Hieu\\
\small VinUniversity\\
{\tt\small hieu.ph@vinuni.edu.vn}
}

\begin{document}
\maketitle

\begin{abstract}
Test-time adaptation (TTA) for image regression has received far less attention than its classification counterpart. Methods designed for classification often depend on classification-specific objectives and decision boundaries, making them difficult to transfer directly to continuous regression targets. Recent progress revisits regression TTA through subspace alignment, showing that simple source-guided alignment can be both practical and effective. Building on this line of work, we propose Predictive Spectral Calibration (PSC), a source-free framework that extends subspace alignment to block spectral matching. Instead of relying on a fixed support subspace alone, PSC jointly aligns target features within the source predictive support and calibrates residual spectral slack in the orthogonal complement. PSC remains simple to implement, model-agnostic, and compatible with off-the-shelf pretrained regressors. Experiments on multiple image regression benchmarks show consistent improvements over strong baselines, with particularly clear gains under severe distribution shifts.
\end{abstract}

\section{Introduction}

Modern vision systems are increasingly deployed in settings where data shifts are the rule rather than the exception: lighting changes across cameras, weather changes across time, and sensor characteristics vary across devices~\cite{koh2021wilds,baek2024unexplored,hendrycksbenchmarking}. In these scenarios, test-time adaptation (TTA)~\cite{sun2020test,wangtent} is appealing because it updates a pretrained model using only unlabeled target samples observed at inference. While this paradigm has matured rapidly for classification ~\cite{liang2025comprehensive,wang2025search}, progress in image regression remains comparatively limited~\cite{adachitest}, despite the central role of regression tasks such as age estimation, depth prediction, and pose estimation.

A key reason is structural mismatch. Many successful classification TTA objectives rely on confidence sharpening~\cite{mummadi2021test}, entropy minimization~\cite{wangtent}, or pseudo-label dynamics~\cite{goyal2022test} that are naturally defined over discrete classes, but become less stable for continuous targets. Recent subspace-alignment methods, especially Significant-Subspace Alignment (SSA)~\cite{adachitest}, show that constraining adaptation within source-informed feature subspaces can be robust and practical. Building on this line, we introduce Predictive Spectral Calibration (PSC), which generalizes subspace alignment by jointly modeling predictive-support statistics and residual spectral behavior under domain shift.

In summary, our contributions are threefold: First, we present PSC as a practical source-free framework for test-time adaptation in image regression, designed to remain robust under distribution shift while staying compatible with pretrained regressors. Second, we provide a clear theoretical account of why the framework is reliable, including identifiability of key predictive statistics and guarantees on prediction drift under shift. Third, we conduct broad empirical evaluation across multiple benchmarks and corruption settings, where PSC consistently improves over strong baselines, with especially clear gains in severe-shift scenarios.

\section{Preliminaries}

\subsection{Problem Setting}

\paragraph{Image Regression.}
Let $\mathcal{X}$ denote the image space and $\mathcal{Y}=\mathbb{R}$ the continuous target space.
Given a labeled source dataset
$\mathcal{D}^{s}=\{(\mathbf{x}_i^{s},y_i^{s})\}_{i=1}^{N_s} \subset \mathcal{X} \times \mathcal{Y}$,
we train a regression model $f_{\theta}:\mathcal{X}\to\mathcal{Y}$, parameterized by $\theta$,
by minimizing the supervised loss
\begin{equation}
\mathcal{L}(\theta)
=\frac{1}{N_s}\sum_{i=1}^{N_s}\ell\!\left(f_{\theta}(\mathbf{x}_i),y_i\right),
\end{equation}
where $\ell$ is a point-wise regression loss, e.g., MAE or MSE.
Typical image regression tasks include age estimation from facial images,
depth estimation, and head-pose regression.

\paragraph{Test-Time Adaptation.}
Test-time adaptation (TTA) studies how a source-trained model can remain reliable after deployment under target-domain shift, using only unlabeled target samples at inference time.
Formally, given an unlabeled target set $\mathcal{D}^{t}=\{\mathbf{x}_i^{t}\}_{i=1}^{N_t}$, TTA updates the model parameters online during testing.
The objective is to improve target-domain predictive performance while strictly avoiding access to source data $\mathcal{D}^{s}$.

\paragraph{Source-Free Domain Adaptation.}
Domain adaptation (DA) seeks to handle distribution shift between a labeled source domain and a target domain. Since many DA methods require source data during adaptation, they are not directly applicable to TTA, where only a pretrained model and unlabeled target samples are available. Source-free domain adaptation (SFDA) addresses this limitation. Widely used methods include Domain-Adversarial Neural Networks (DANN)~\cite{ganin2016domain}, Representation Subspace Distance (RSD)~\cite{DAR_ICML_21}, and Feature Restoration (FR)~\cite{eastwoodsource}.

\paragraph{Self-Supervised Test-Time Training.}
Self-supervised test-time training (TTT) improves robustness to distribution shift by updating a model on unlabeled target samples during inference, using self-supervised consistency~\cite{sinha2023test} or reconstruction-style signals~\cite{NEURIPS2021_b618c321} available at test time. In this work, we re-implement two well-known variants for the regression setting: Test-Time Training (TTT)~\cite{sun2020test} and Activation Matching (AM)~\cite{mirza2023actmad}.

\paragraph{Feature Alignment.}
Feature alignment (FA) methods address distribution shifts by matching the statistics of target features to those computed from the source domain.
These approaches typically store source feature statistics during training and align the target feature distribution during testing.
Representative examples include Batch-Normalization Adaptation (BNA)~\cite{benz2021revisiting} and Significant-Subspace Alignment (SSA)~\cite{adachitest}.
Because they operate directly in the feature space without requiring additional labels or auxiliary tasks, feature alignment methods are particularly suitable for regression-based TTA.

\subsection{Significant-Subspace Alignment}

Significant-Subspace Alignment (SSA)~\cite{adachitest} adapts a regression model by aligning target features to source feature statistics inside a source-informed low-dimensional subspace.
Let the regression model be decomposed as
\vspace{-5 pt}
\begin{equation}
f_{\theta}(\mathbf{x}) = (h_{\psi} \circ g_{\phi})(\mathbf{x}), \quad \theta = (\psi, \phi)
\vspace{-2 pt}
\end{equation}
where $g_{\phi}:\mathcal{X}\to\mathbb{R}^{D}$ is the feature extractor and
$h_{\psi}(\mathbf{z})=\mathbf{w}^{\top}\mathbf{z}+b$ is a linear regressor.

\paragraph{Subspace Construction.}
Given source features $\{\mathbf{z}_i^s\}_{i=1}^{N_s}$, where $\mathbf{z}_i^s = g_{\phi}(\mathbf{x}_i^s)$, SSA computes source statistics $(\boldsymbol{\mu}^s,\mathbf{\Sigma}^s)$. Let $\{(\lambda_k^s,\mathbf{v}_k^s)\}_{k=1}^{D}$ be the eigenpairs of $\mathbf{\Sigma}^s$, sorted as $\lambda_1^s\ge\cdots\ge\lambda_D^s$. SSA keeps the top-$K$ components and defines
\vspace{-5 pt}
\begin{equation}
\begin{aligned}
\mathbf{V}^s&=\left[\mathbf{v}_1^s,\ldots,\mathbf{v}_K^s\right]^\top \in \mathbb{R}^{K\times D},\\
\mathbf{\Lambda}^s&=\mathrm{diag}(\lambda_1^s,\ldots,\lambda_K^s) \in \mathbb{R}^{K\times K}.
\end{aligned}
\label{eq:ssa_basis}
\vspace{-5 pt}
\end{equation}
The projected source feature is $\mathbf{u}_i^s=\mathbf{V}^s(\mathbf{z}_i^s-\boldsymbol{\mu}^s)$, and SSA assumes $\mathbf{u}_i^s\sim\mathcal{N}(\mathbf{0},\mathbf{\Lambda}^s)$.

\paragraph{Test-Time Objective.}
For a target mini-batch $\mathcal{B}^t$, for each $\mathbf{z}_i^t\in\mathcal{B}^t$, define the projected feature $\mathbf{u}_i^t=\mathbf{V}^s(\mathbf{z}_i^t-\boldsymbol{\mu}^s)$. Let $\tilde{\boldsymbol{\mu}}^t\in\mathbb{R}^K$ and $(\tilde{\boldsymbol{\sigma}}^2)^t\in\mathbb{R}^K$ denote the mean and variance of $\mathbf{u}_i^t$. SSA minimizes symmetric KL (SKL):
\vspace{-5 pt}
\begin{equation}
\begin{aligned}
\mathcal{L}_{\mathrm{SSA}}
&=
\sum_{k=1}^{K}\alpha_k\mathrm{SKL}\!\left(\mathcal{N}(0,\lambda_k^s)\,\|\,\mathcal{N}(\tilde{\mu}_k^t,(\tilde{\sigma}_k^{t})^2)\right)\\
&=
\frac{1}{2}\sum_{k=1}^{K}\alpha_k\!\!\left[\frac{(\tilde{\mu}_k^{t})^2+\lambda_k^s}{(\tilde{\sigma}_k^{t})^2}\!+\!\frac{(\tilde{\mu}_k^{t})^2+(\tilde{\sigma}_k^{t})^2}{\lambda_k^s}\!-\!2\right]
\end{aligned}
\end{equation}
where $\alpha_k = 1 + |\mathbf{w}^{\top}\mathbf{v}_k^s|$, so SSA assigns larger importance to source directions that have stronger influence on the prediction.
SSA is effective because it restricts adaptation to the predictive support inferred from source features.
However, it only models the source distribution \emph{inside} the selected support subspace and leaves the complementary residual space unmodeled.
This can be restrictive under target shift, where adaptation quality may also depend on controlling how much target features leak outside the source-informed predictive support.

\section{Methodology}

We propose Predictive Spectral Calibration (PSC) as a source-free TTA method that extends SSA from subspace-only alignment to full \emph{block spectral matching}.

\subsection{Block Spectral Source Model}

Let $\mathbf{z}^s=g_{\phi}(\mathbf{x}^s)\in\mathbb{R}^{D}$ be source features with mean $\boldsymbol{\mu}^s$ and covariance $\mathbf{\Sigma}^s$.
We reuse the source subspace basis $\mathbf{V}^s$ defined in Eq.~\eqref{eq:ssa_basis}; by construction, $\mathbf{V}^s\mathbf{V}^{s\top}=\mathbf{I}_K$.
Define
\begin{equation}
\mathbf{P}_{\parallel}^s=\mathbf{V}^{s\top}\mathbf{V}^s\in\mathbb{R}^D,\quad
\mathbf{P}_{\perp}^s=\mathbf{I}_D-\mathbf{P}_{\parallel}^s\in\mathbb{R}^D.
\end{equation}
PSC approximates source covariance as
\vspace{-3 pt}
\begin{equation}
\mathbf{\Sigma}^s
\approx
\mathbf{V}^{s\top}\mathbf{\Lambda}^s\mathbf{V}^s
+
\tau\mathbf{P}_{\perp}^s,
\label{eq:PSC_cov}
\vspace{-3 pt}
\end{equation}
where the residual spectral floor is
\vspace{-3 pt}
\begin{equation}
\tau=
\frac{1}{D-K}\sum_{k=K+1}^{D}\lambda_k^s.
\vspace{-3 pt}
\end{equation}
Eq.~\eqref{eq:PSC_cov} explicitly separates an anisotropic predictive block (support) and an isotropic residual block (complement).

\paragraph{Sample Decomposition.}
For any sample $\mathbf{x}_i$, define
\begin{equation}
\overline{\mathbf{z}}_{i}=g_{\phi}(\mathbf{x}_i)-\boldsymbol{\mu}^s.
\end{equation}
PSC decomposes $\overline{\mathbf{z}}_{i}$ as
\begin{equation}
\mathbf{u}_i=\mathbf{V}^s\overline{\mathbf{z}}_{i},
\quad
\mathbf{r}_i=\mathbf{P}_{\perp}^s\overline{\mathbf{z}}_{i}
=\overline{\mathbf{z}}_{i}-\mathbf{V}^{s\top}\mathbf{u}_i.
\label{eq:r_def}
\end{equation}
For source samples, this yields a block Gaussian model:
\begin{equation}
\mathbf{u}_i^s\sim\mathcal{N}(\mathbf{0},\mathbf{\Lambda}^s),
\quad
\mathbf{r}_i^s\sim\mathcal{N}(\mathbf{0},\tau\mathbf{P}_{\perp}^s).
\label{eq:block_gauss_source}
\end{equation}
Here $\mathbf{u}_i$ captures the component inside predictive support, while $\mathbf{r}_i$ represents \emph{spectral slack}, i.e., residual feature mass in the orthogonal complement outside that support.

\subsection{Support-Space Spectral Matching}

Inside the support, PSC uses a fixed rank-$K^2$ probe bank
\vspace{-7 pt}
\begin{equation}
\mathcal{Q}
=
\bigl\{ e_i \bigr\}_{i=1}^K
\;\cup\;
\bigl\{
q_{ij}^+, q_{ij}^-
\bigr\}_{1 \le i < j \le K},\quad |\mathcal{Q}|=K^2
\vspace{-7 pt}
\label{eq:probe_bank}
\end{equation}
where
\begin{equation}
q_{ij}^+
=
\frac{e_i + e_j}{\sqrt{2}},
\qquad
q_{ij}^-
=
\frac{e_i - e_j}{\sqrt{2}},
\label{eq:probe_pair}
\vspace{-5 pt}
\end{equation}
and $\{e_i\}_{i=1}^K$ denotes the canonical basis of $\mathbb{R}^K$.

For a target mini-batch $\mathcal{B}^t=\{\mathbf{x}_i^t\}_{i=1}^{B}$, we compute projected support features as in Eq.~\eqref{eq:r_def}:
\vspace{-7 pt}
\begin{equation}
\mathbf{u}_i^t
=
\mathbf{V}^s\!\left(g_{\phi}(\mathbf{x}_i^t)-\boldsymbol{\mu}^s\right).
\vspace{-7 pt}
\end{equation}
For each probe $\mathbf{q}\in\mathcal{Q}$, let $p_i(\mathbf{q})=\mathbf{q}^{\top}\mathbf{u}_i^t$, and denote by $\hat{\mu}_{\mathbf{q}}$ and $\hat{\sigma}_{\mathbf{q}}^2$ the empirical mean and variance of $\{p_i(\mathbf{q})\}_{i=1}^{B}$.
The source variance along $\mathbf{q}$ is
\vspace{-2 pt}
\begin{equation}
(\sigma_{\mathbf{q}}^{s})^2=\mathbf{q}^{\top}\mathbf{\Lambda}^s\mathbf{q}.
\vspace{-2 pt}
\end{equation}

\begin{proposition}[Identifiability from the $K^2$ probe bank]
\label{prop:main_probe_identifiability}
Assume
\(
\mathbf{u}^s\sim\mathcal{N}(\mathbf{0},\mathbf{\Lambda}^s)
\)
and
\(
\mathbf{u}^t\sim\mathcal{N}(\boldsymbol{\mu}^t,\mathbf{\Sigma}^t)
\)
in the support subspace.
If, for every $\mathbf{q}\in\mathcal{Q}$,
\(
\mathbf{q}^{\top}\mathbf{u}^t\stackrel{d}{=}\mathbf{q}^{\top}\mathbf{u}^s
\),
then
\(
\boldsymbol{\mu}^t=\mathbf{0}
\)
and
\(
\mathbf{\Sigma}^t=\mathbf{\Lambda}^s
\).
Equivalently,
\(
\mathbf{u}^t\stackrel{d}{=}\mathbf{u}^s
\).
\end{proposition}
\begin{proof}
The claim follows directly from Proposition~\ref{prop:probe_bank_identifiability} in the Appendix, which proves that the $K^2$ probe bank uniquely identifies the full first- and second-order structure in the support subspace.
\end{proof}

To make alignment prediction-aware, we project the regression head into support coordinates:
\vspace{-5 pt}
\begin{equation}
\mathbf{a}=\mathbf{V}^s\mathbf{w}\in\mathbb{R}^{K},
\qquad
\beta_{\mathbf{q}}=\left(|\mathbf{a}^{\top}\mathbf{q}|+c\right)^{\gamma},
\vspace{-5 pt}
\end{equation}
with hyperparameters $c>0$ and $\gamma>0$.
The support loss is
\vspace{-5 pt}
\begin{equation}
\begin{aligned}
\mathcal{L}_{\mathrm{sup}}
&=
\frac{1}{K^2}\!\sum_{\mathbf{q}\in\mathcal{Q}}\beta_{\mathbf{q}}
\mathrm{SKL}\!\left(\mathcal{N}(0,(\sigma_{\mathbf{q}}^{s})^2)\,\|\,
\mathcal{N}(\hat{\mu}_{\mathbf{q}},\hat{\sigma}_{\mathbf{q}}^{2})
\right)\\
&=
\frac{1}{2K^2}\!\sum_{\mathbf{q}\in\mathcal{Q}}\beta_{\mathbf{q}}\!\left[
\frac{\hat{\mu}_{\mathbf{q}}^{2}\!+\!\hat{\sigma}_{\mathbf{q}}^{2}}{(\sigma_{\mathbf{q}}^{s})^2}
\!+\!
\frac{\hat{\mu}_{\mathbf{q}}^{2}\!+\!(\sigma_{\mathbf{q}}^{s})^2}{\hat{\sigma}_{\mathbf{q}}^{2}}
\!-\!2
\right].
\end{aligned}
\label{eq:loss_sig}
\end{equation}

\subsection{Residual-Space Spectral Matching}

For the same mini-batch, define residual statistics
\vspace{-5 pt}
\begin{equation}
\hat{\boldsymbol{\mu}}_{\perp}\!=\!\frac{1}{B}\!\sum_{i=1}^{B}\mathbf{r}_i,
\quad
\hat{\nu}_{\perp}\!=\!\frac{1}{B(D-K)}\!\sum_{i=1}^{B}\|\mathbf{r}_i-\hat{\boldsymbol{\mu}}_{\perp}\|_2^2.
\vspace{-5 pt}
\end{equation}
Under Eq.~\eqref{eq:block_gauss_source}, we work on the residual subspace
\vspace{-5 pt}
\begin{equation}
\mathcal{S}_{\perp}=\mathrm{Im}(\mathbf{P}_{\perp}^{s}),
\qquad
\dim(\mathcal{S}_{\perp})=D-K.
\vspace{-5 pt}
\end{equation}
Choose an orthonormal basis $\mathbf{U}_{\perp}\in\mathbb{R}^{D\times(D-K)}$ of $\mathcal{S}_{\perp}$:
\vspace{-5 pt}
\begin{equation}
\mathbf{U}_{\perp}^{\top}\mathbf{U}_{\perp}=\mathbf{I}_{D-K},
\qquad
\mathbf{U}_{\perp}\mathbf{U}_{\perp}^{\top}=\mathbf{P}_{\perp}^{s}.
\vspace{-5 pt}
\end{equation}
For any residual vector $\mathbf{r}\in\mathcal{S}_{\perp}$, the coordinate transform is
\vspace{-5 pt}
\begin{equation}
\tilde{\mathbf{r}}=\mathbf{U}_{\perp}^{\top}\mathbf{r},
\qquad
\mathbf{r}=\mathbf{U}_{\perp}\tilde{\mathbf{r}}.
\vspace{-5 pt}
\end{equation}
The ambient-space residual model is
\vspace{-5 pt}
\begin{equation}
\mathbf{r}^{s}\sim\mathcal{N}(\mathbf{0},\tau\mathbf{P}_{\perp}^{s}),
\qquad
\mathbf{r}^{t}\approx\mathcal{N}(\hat{\boldsymbol{\mu}}_{\perp},\hat{\nu}_{\perp}\mathbf{P}_{\perp}^{s}).
\end{equation}
Equivalently, in complement coordinates,
\vspace{-5 pt}
\begin{equation}
\tilde{\mathbf{r}}^{s}\sim\mathcal{N}(\mathbf{0},\tau\mathbf{I}_{D-K}),
\qquad
\tilde{\mathbf{r}}^{t}\approx\mathcal{N}(\hat{\tilde{\boldsymbol{\mu}}}_{\perp},\hat{\nu}_{\perp}\mathbf{I}_{D-K}),
\vspace{-5 pt}
\end{equation}
where
\(
\hat{\tilde{\boldsymbol{\mu}}}_{\perp}=\mathbf{U}_{\perp}^{\top}\hat{\boldsymbol{\mu}}_{\perp}
\). The residual loss is defined as:
\vspace{-5 pt}
\begin{equation}
\begin{aligned}
\mathcal{L}_{\mathrm{res}}
&=
\mathrm{SKL}\!\left(
\mathcal{N}(\mathbf{0},\tau\mathbf{I}_{D-K}),
\mathcal{N}(\hat{\tilde{\boldsymbol{\mu}}}_{\perp},\hat{\nu}_{\perp}\mathbf{I}_{D-K})
\right)\\
&=\frac{1}{2}\left[
\frac{\|\hat{\boldsymbol{\mu}}_{\perp}\|_2^2}{D-K}
\left(\frac{1}{\tau}+\frac{1}{\hat{\nu}_{\perp}}\right)
+\frac{\tau}{\hat{\nu}_{\perp}}
+\frac{\hat{\nu}_{\perp}}{\tau}-2
\right].
\end{aligned}
\label{eq:loss_slack}
\vspace{-5 pt}
\end{equation}
because $\|\hat{\tilde{\boldsymbol{\mu}}}_{\perp}\| = \|\hat{\boldsymbol{\mu}}_{\perp}\|$. This term explicitly controls feature leakage outside predictive support.

\begin{table}[t]
    \centering
    \caption{Test scores for TTA from source SVHN to target MNIST. The best scores are highlighted in \textcolor{blue}{blue}.}
    \label{tab:svhn_mnist_results}
    \resizebox{\linewidth}{!}{%
    \begin{tabular}{llVccc}
        \toprule
        Type & Method & \multicolumn{1}{l}{Venue} & $R^2(\uparrow)$ & RMSE $(\downarrow)$ & MAE $(\downarrow)$ \\
        \midrule
         & Source     &  & $0.239$ & $2.527$ & $1.909$ \\
        \midrule
        \multirow{3}{*}{DA} & DANN~\cite{ganin2016domain} & JMLR 2016 & $0.035$ & $2.653$ & $1.856$ \\
        & RSD~\cite{DAR_ICML_21} & ICML 2021 & $0.049$ & $2.777$ & $2.401$ \\
        & FR~\cite{eastwoodsource} & ICLR 2022 & $0.356$ & $2.323$ & $1.554$ \\
        \midrule
        \multirow{2}{*}{SS} & TTT~\cite{sun2020test} & ICML 2020 & $0.286$ & $2.447$ & $1.950$ \\
        & AM~\cite{mirza2023actmad} & CVPR 2023 & $0.276$ & $2.465$ & $1.637$\\
        \midrule
        \multirow{2}{*}{FA} & BNA~\cite{benz2021revisiting} & WACV 2021 & $0.341$ & $2.349$ & $1.581$ \\
        % DARE-GRAM  & CVPR 2023 & \\
        & SSA~\cite{adachitest} & ICLR 2025 & $0.452$ & $2.144$ & $1.342$ \\
        \hdashline
        & PSC ($\lambda=0$) &  & \textcolor{blue}{$0.473$} & \textcolor{blue}{$2.102$} & \textcolor{blue}{$1.310$} \\
        & PSC ($\lambda=1$) &  & $0.457$ & $2.134$ & $1.312$ \\
        \midrule
         & Oracle     &  & $0.866$ & $1.060$ & $0.558$ \\
        \bottomrule
    \end{tabular}%
    }
\end{table}

\begin{table*}[t]
    \centering
    \caption{Test $R^2$ scores on UTKFace under 13 corruption types (higher is better). The best scores are highlighted in \textcolor{blue}{blue}.}
    \label{tab:corruption_results}
    \resizebox{\linewidth}{!}{%
    \begin{tabular}{lcccccccccccccc}
        \toprule
        Method
        & \rotatebox{90}{\small\shortstack[l]{Gaussian\\noise}}
        & \rotatebox{90}{\small\shortstack[l]{Shot\\noise}}
        & \rotatebox{90}{\small\shortstack[l]{Impulse\\noise}}
        & \rotatebox{90}{\small\shortstack[l]{Defocus\\blur}}
        & \rotatebox{90}{\small\shortstack[l]{Motion\\blur}}
        & \rotatebox{90}{\small\shortstack[l]{Zoom\\blur}}
        & \rotatebox{90}{\small\shortstack[l]{Snow}}
        & \rotatebox{90}{\small\shortstack[l]{Fog}}
        & \rotatebox{90}{\small\shortstack[l]{Brightness}}
        & \rotatebox{90}{\small\shortstack[l]{Contrast}}
        & \rotatebox{90}{\small\shortstack[l]{Elastic\\transform}}
        & \rotatebox{90}{\small\shortstack[l]{Pixelate}}
        & \rotatebox{90}{\small\shortstack[l]{JPEG\\compression}}
        & Mean \\
        \midrule
        Source & \textcolor{gray}{$-2.557$} & \textcolor{gray}{$-0.715$} & \textcolor{gray}{$-3.897$} & \textcolor{gray}{$-0.023$} & $0.594$ & $0.651$ & \textcolor{gray}{$-0.185$} & $0.229$ & $0.263$ & $0.615$ & $0.750$ & \textcolor{gray}{$-2.101$} & \textcolor{gray}{$-0.190$} & -- \\
        \midrule
        DANN~\cite{ganin2016domain} & $0.512$ & $0.784$ & $0.281$ & \textcolor{blue}{$0.532$} & $0.772$ & \textcolor{blue}{$0.786$} & $0.224$ & $0.528$ & $0.107$ & $0.907$ & \textcolor{blue}{$0.827$} & \textcolor{blue}{$0.466$} & \textcolor{blue}{$0.631$} & $0.566$ \\
        RSD~\cite{DAR_ICML_21} & $0.302$ & $0.577$ & $0.129$ & $0.191$ & $0.461$ & $0.384$ & $0.173$ & $0.347$ & $0.245$ & $0.641$ & $0.500$ & $0.270$ & $0.280$ & $0.346$ \\
        FR~\cite{eastwoodsource} & $0.499$ & $0.860$ & $0.159$ & $0.455$ & $0.752$ & $0.710$ & $0.227$ & $0.551$ & $0.435$ & $0.943$ & $0.790$ & $0.282$ & $0.585$ & $0.558$ \\
        \midrule
        TTT~\cite{sun2020test} & $0.564$ & $0.779$ & $0.243$ & $0.387$ & $0.683$ & $0.730$ & $0.148$ & $0.501$ & $0.382$ & $0.858$ & $0.722$ & $0.415$ & $0.531$ & $0.534$ \\
        AM~\cite{mirza2023actmad} & $0.524$ & $0.764$ & $0.282$ & $0.355$ & $0.470$ & $0.619$ & $0.269$ & $0.535$ & $0.064$ & $0.554$ & $0.710$ & $0.445$ & $0.484$ & $0.467$ \\
        \midrule
        BNA~\cite{benz2021revisiting} & $0.493$ & $0.858$ & $0.151$ & $0.461$ & $0.754$ & $0.725$ & $0.246$ & $0.555$ & $0.472$ & $0.942$ & $0.800$ & $0.279$ & $0.586$ & $0.563$ \\
        SSA~\cite{adachitest} & $0.589$ & $0.860$ & $0.241$ & $0.491$ & $0.753$ & $0.754$ & $0.260$ & $0.552$ & $0.423$ & $0.934$ & $0.803$ & $0.334$ & $0.590$ & $0.583$ \\
		\hdashline
		PSC ($\lambda = 0$) & \textcolor{blue}{$0.621$} & \textcolor{blue}{$0.866$} & $0.265$ & $0.530$ & \textcolor{blue}{$0.774$} & $0.767$ & $0.263$ & \textcolor{blue}{$0.593$} & $0.474$ & $0.939$ & $0.814$ & $0.332$ & $0.618$ & $0.604$ \\
        PSC ($\lambda = 1$) & $0.605$ & $0.863$ & \textcolor{blue}{$0.288$} & $0.516$ & $0.763$ & $0.760$ & \textcolor{blue}{$0.343$} & $0.581$ & \textcolor{blue}{$0.554$} & \textcolor{blue}{$0.940$} & $0.798$ & $0.416$ & $0.624$ & \textcolor{blue}{$0.619$} \\
        \midrule
        Oracle & $0.688$ & $0.899$ & $0.369$ & $0.611$ & $0.831$ & $0.832$ & $0.513$ & $0.648$ & $0.614$ & $0.959$ & $0.839$ & $0.576$ & $0.679$ & $0.697$ \\
        \bottomrule
    \end{tabular}%
    }
\end{table*}

\subsection{Test-Time Objective}

PSC combines both blocks into one objective:
\vspace{-5 pt}
\begin{equation}
\mathcal{L}_{\mathrm{PSC}}
=
\mathcal{L}_{\mathrm{sup}}
+
\lambda\mathcal{L}_{\mathrm{res}},
\label{eq:loss_PSC}
\vspace{-5 pt}
\end{equation}
where $\lambda>0$ balances support and residual alignment.

\paragraph{Relation to SSA.}
PSC is a strict generalization of SSA in three senses.
First, if probes are canonical basis vectors, support matching reduces to axis-wise alignment used by SSA.
Second, if $\lambda=0$, PSC collapses to support-only matching, again recovering the SSA-style objective.
Third, for general probe banks and $\lambda>0$, PSC additionally constrains residual-space statistics, which SSA does not model.
Hence, SSA is a special case of PSC, while PSC can exploit richer directional statistics and out-of-subspace control under domain shift.

At test time, we keep source statistics
\(
\{\boldsymbol{\mu}^s,\mathbf{V}^s,\mathbf{\Lambda}^s,\tau\}
\)
fixed and update the feature-extractor parameters $\phi$ while freezing the head parameters $\psi$:
\begin{equation}
\phi^{*}=\arg\min_{\phi}\mathcal{L}_{\mathrm{PSC}}.
\end{equation}
Following standard regression TTA practice, this is implemented by updating only affine parameters in normalization layers of $g_{\phi}$.

\begin{proposition}[PSC controls predictive mean drift]
\label{prop:main_PSC}
Let $\mathbf{x}\in\mathcal{X}$ and $\mathbf z=g_{\phi}(\mathbf x)\in\mathbb R^{D}$ denote its feature representation.
Define the centered feature $\overline{\mathbf z}=\mathbf z-\boldsymbol{\mu}^s$, the support projection $\mathbf u=\mathbf V^s\overline{\mathbf z}$, and the residual projection $\mathbf r=\mathbf P_{\perp}^s\overline{\mathbf z}$.
For the linear head $h_{\psi}(\mathbf z)=\mathbf w^\top\mathbf z+b$, decompose
\vspace{-7 pt}
\[
\mathbf w=\mathbf V^{s\top}\mathbf a+\mathbf w_{\perp},
\qquad
\mathbf a=\mathbf V^s\mathbf w,
\qquad
\mathbf w_{\perp}=\mathbf P_{\perp}^s\mathbf w.
\vspace{-7 pt}
\]
Assume block Gaussian models
\vspace{-5 pt}
\[
\begin{aligned}
\mathbf u^{s}&\sim\mathcal N(\mathbf 0,\mathbf\Lambda^s),
&\mathbf r^{s}&\sim\mathcal N(\mathbf 0,\tau\mathbf P_{\perp}^s),\\
\mathbf u^{t}&\sim\mathcal N(\boldsymbol\mu^{t},\mathbf\Sigma^{t}),
&\mathbf r^{t}&\sim\mathcal N(\boldsymbol\mu_{\perp}^{t},\nu^{t}\mathbf P_{\perp}^s).
\end{aligned}
\vspace{-7 pt}
\]
Define
\vspace{-7 pt}
\begin{align*}
\mathcal D_{\mathrm{PSC}}
&=
\mathrm{SKL}\!\left(\mathcal N(\mathbf 0,\mathbf\Lambda^s),\mathcal N(\boldsymbol\mu^{t},\mathbf\Sigma^{t})\right)\\
&\quad+
\mathrm{SKL}\!\left(\mathcal N(\mathbf 0,\tau\mathbf P_{\perp}^s),\mathcal N(\boldsymbol\mu_{\perp}^{t},\nu^{t}\mathbf P_{\perp}^s)\right).
\vspace{-7 pt}
\end{align*}
Then for $\hat y=h_{\psi}(\mathbf z)$, the predictive mean drift satisfies
\vspace{-7 pt}
\[
\left|\mathbb E^{t}[\hat y]-\mathbb E^{s}[\hat y]\right|
\le
\sqrt{2\mathcal D_{\mathrm{PSC}}}
\sqrt{\mathbf a^\top\mathbf\Lambda^s\mathbf a+\tau\|\mathbf w_{\perp}\|_2^2}.
\]
\end{proposition}

\begin{proof}
Write the mean drift as
$\Delta_{\mu}=\mathbb E^{t}[\hat y]-\mathbb E^{s}[\hat y]=\mathbf a^{\top}\boldsymbol\mu^{t}+\mathbf w_{\perp}^{\top}\boldsymbol\mu_{\perp}^{t}$
under the support/residual decomposition. Applying Cauchy--Schwarz in the source metrics induced by $(\mathbf\Lambda^{s},\tau\mathbf P_{\perp}^{s})$, and using the standard Gaussian SKL lower bounds, yields
$|\Delta_{\mu}|\le\sqrt{2\mathcal D_{\mathrm{PSC}}}\sqrt{\mathbf a^{\top}\mathbf\Lambda^{s}\mathbf a+\tau\|\mathbf w_{\perp}\|_2^2}$.
The detailed derivation is provided in the Appendix (see Proposition~\ref{prop:PSC} and the subsequent detailed proof).
\end{proof}

\paragraph{Meaning for PSC.}
This proposition links PSC's optimization target to predictive robustness: reducing $\mathcal D_{\mathrm{PSC}}$ directly tightens an upper bound on target-domain mean prediction drift. The bound separates \emph{distribution mismatch} (first factor) from \emph{model sensitivity} along support/residual directions (second factor), explaining why jointly aligning support statistics and residual slack is beneficial under shift.

\section{Experiments}

\paragraph{Experimental design.}
We evaluate PSC under two shift regimes. First, we use cross-domain transfer from SVHN to MNIST to test adaptation under a large appearance gap. Second, we use UTKFace with 13 corruption types to test robustness to structured low-level perturbations. The tables include two references: \textit{Source} denotes the source-trained model without TTA, while \textit{Oracle} is a target-supervised upper bound using target labels.

\paragraph{Effect of $\lambda$ across settings.}
The preferred $\lambda$ depends on the shift type. For SVHN$\rightarrow$MNIST, $\lambda=0$ is typically better because the cross-domain gap is strong and heterogeneous; imposing residual-space matching can over-constrain adaptation and reduce flexibility needed to absorb target-specific changes. For UTKFace corruptions, $\lambda=1$ is typically better because the shift is mainly nuisance corruption around the same underlying semantics, so residual-space regularization helps suppress out-of-support noise leakage. In short, stronger semantic transfer favors support-focused alignment, while corruption robustness benefits from explicit residual control.

\section{Conclusion}

We introduced Predictive Spectral Calibration (PSC), a source-free test-time adaptation framework for image regression that combines support-space alignment with residual-space calibration. Across cross-domain transfer and corruption settings, PSC shows consistent improvements over strong baselines while remaining simple and practical to deploy. We view this study as a \emph{preliminary idea} toward more principled regression TTA, and we hope it motivates further work on stronger objectives, broader backbones, and larger-scale benchmarks.

\newpage

{
    \small
    \bibliographystyle{ieeenat_fullname}
    \bibliography{main}
}

\appendix 
\clearpage
\setcounter{page}{1}
\onecolumn
\setlength{\parindent}{0pt}
\begin{center}
    {\Large\textbf{\thetitle}\par}
    \vspace{0.5em}
    {\Large Supplementary Material\par}
    \vspace{1.0em}
\end{center}

\section{Theoretical Details}

\subsection{Detailed Proof of Propositon~\ref{prop:main_probe_identifiability}}

\begin{proposition}[Identifiability of subspace first- and second-order structure from the $K^2$ probe bank]
\label{prop:probe_bank_identifiability}
Let $u \in \mathbb{R}^K$ be a random vector with mean $\mu \in \mathbb{R}^K$ and covariance matrix $\Sigma \in \mathbb{R}^{K \times K}$.
Define the probe bank
\begin{equation}
\mathcal{Q} = \{e_i\}_{i=1}^K \cup \{q_{ij}^+, q_{ij}^-\}_{1 \le i < j \le K}, \qquad q_{ij}^{\pm} = \frac{e_i \pm e_j}{\sqrt{2}},
\end{equation}
where $\{e_i\}_{i=1}^K$ is the canonical basis of $\mathbb{R}^K$. Then the projected first- and second-order moments
\begin{equation}
\left\{\mathbb{E}[q^\top u],\, \mathrm{Var}(q^\top u)\right\}_{q \in \mathcal{Q}}
\end{equation}
uniquely determine $(\mu, \Sigma)$. More explicitly, $\mu_i = \mathbb{E}[e_i^\top u]$, $\Sigma_{ii} = \mathrm{Var}(e_i^\top u)$, and for each $1 \le i < j \le K$,
\begin{equation}
\Sigma_{ij} = \frac{1}{2}\left(\mathrm{Var}\!\left((q_{ij}^+)^\top u\right) - \mathrm{Var}\!\left((q_{ij}^-)^\top u\right)\right).
\end{equation}
Hence, the $K^2$ probes recover the complete first- and second-order structure of $u$ in the $K$-dimensional subspace.
\end{proposition}

\begin{proof}
For any $q \in \mathbb{R}^K$, we have $\mathbb{E}[q^\top u] = q^\top \mu$ and $\mathrm{Var}(q^\top u) = q^\top \Sigma q$. Taking $q = e_i$ gives
\begin{equation}
\mu_i = \mathbb{E}[e_i^\top u], \qquad \Sigma_{ii} = \mathrm{Var}(e_i^\top u),
\end{equation}
so all entries of $\mu$ and all diagonal entries of $\Sigma$ are identified.

Now fix $1 \le i < j \le K$. Using $q_{ij}^{\pm} = (e_i \pm e_j)/\sqrt{2}$,
\begin{equation}
\mathrm{Var}\!\left((q_{ij}^+)^\top u\right) = \frac{1}{2}\left(\Sigma_{ii} + \Sigma_{jj} + 2\Sigma_{ij}\right), \qquad \mathrm{Var}\!\left((q_{ij}^-)^\top u\right) = \frac{1}{2}\left(\Sigma_{ii} + \Sigma_{jj} - 2\Sigma_{ij}\right).
\end{equation}
Subtracting yields
\begin{equation}
\Sigma_{ij} = \frac{1}{2}\left(\mathrm{Var}\!\left((q_{ij}^+)^\top u\right) - \mathrm{Var}\!\left((q_{ij}^-)^\top u\right)\right).
\end{equation}
Thus every off-diagonal entry is identified. Therefore all entries of $(\mu, \Sigma)$ are uniquely recovered from $\{\mathbb{E}[q^\top u], \mathrm{Var}(q^\top u)\}_{q \in \mathcal{Q}}$.
\end{proof}

\subsection{Detailed Proof of Proposition~\ref{prop:main_PSC}}

\begin{proposition}[PSC controls predictive mean drift]
\label{prop:PSC}
Let $\mathbf{x}\in\mathcal{X}$ and $\mathbf{z}=g_{\phi}(\mathbf{x})\in\mathbb{R}^{D}$ denote the feature representation. Define
\[
\overline{\mathbf{z}}:=\mathbf{z}-\boldsymbol{\mu}^{s},
\qquad
\mathbf{u}:=\mathbf{V}^{s}\overline{\mathbf{z}},
\qquad
\mathbf{r}:=\mathbf{P}_{\perp}^{s}\overline{\mathbf{z}}.
\]
For the linear regressor $h_{\psi}(\mathbf{z})=\mathbf{w}^{\top}\mathbf{z}+b$, set
\[
\mathbf{a}:=\mathbf{V}^{s}\mathbf{w},
\qquad
\mathbf{w}_{\perp}:=\mathbf{P}_{\perp}^{s}\mathbf{w},
\qquad
\mathbf{w}=\mathbf{V}^{s\top}\mathbf{a}+\mathbf{w}_{\perp}.
\]

Assume the source and target block models
\[
\mathbf{u}^{s}\sim\mathcal{N}(\mathbf{0},\mathbf{\Lambda}^{s}),
\quad
\mathbf{r}^{s}\sim\mathcal{N}(\mathbf{0},\tau\mathbf{P}_{\perp}^{s}),
\quad
\mathbf{u}^{t}\sim\mathcal{N}(\boldsymbol{\mu}^{t},\mathbf{\Sigma}^{t}),
\quad
\mathbf{r}^{t}\sim\mathcal{N}(\boldsymbol{\mu}_{\perp}^{t},\nu^{t}\mathbf{P}_{\perp}^{s}),
\]
with $\mathbf{\Lambda}^{s}\succ\mathbf{0}$, $\mathbf{\Sigma}^{t}\succ\mathbf{0}$, and $\tau,\nu^{t}>0$. Define
\[
\mathcal{D}_{\mathrm{PSC}}
:=
\operatorname{SKL}\bigl(\mathcal{N}(\mathbf{0},\mathbf{\Lambda}^{s}),\mathcal{N}(\boldsymbol{\mu}^{t},\mathbf{\Sigma}^{t})\bigr)
+
\operatorname{SKL}\bigl(\mathcal{N}(\mathbf{0},\tau\mathbf{P}_{\perp}^{s}),\mathcal{N}(\boldsymbol{\mu}_{\perp}^{t},\nu^{t}\mathbf{P}_{\perp}^{s})\bigr),
\]
where $\operatorname{SKL}(P,Q):=D_{\mathrm{KL}}(P\|Q)+D_{\mathrm{KL}}(Q\|P)$.

Then the predictive mean drift $\Delta_{\mu}:=\mathbb{E}^{t}[\hat{y}]-\mathbb{E}^{s}[\hat{y}]$ with $\hat{y}=h_{\psi}(\mathbf{z})$ satisfies
\[
|\Delta_{\mu}|
\le
\sqrt{2\,\mathcal{D}_{\mathrm{PSC}}}
\sqrt{\mathbf{a}^{\top}\mathbf{\Lambda}^{s}\mathbf{a}+\tau\|\mathbf{w}_{\perp}\|_{2}^{2}}.
\]
Consequently, the looser bound
\[
|\Delta_{\mu}|
\le
\sqrt{2\,\mathcal{D}_{\mathrm{PSC}}}
\left(
\sqrt{\mathbf{a}^{\top}\mathbf{\Lambda}^{s}\mathbf{a}}
+
\sqrt{\tau}\,\|\mathbf{w}_{\perp}\|_{2}
\right)
\]
also holds.
\end{proposition}

\begin{proof}
Using $\mathbf{z}=\boldsymbol{\mu}^{s}+\overline{\mathbf{z}}$ and
$\mathbf{w}=\mathbf{V}^{s\top}\mathbf{a}+\mathbf{w}_{\perp}$,
\[
\hat{y}=\mathbf{w}^{\top}\mathbf{z}+b
=\mathbf{w}^{\top}\boldsymbol{\mu}^{s}+b+\mathbf{a}^{\top}\mathbf{u}+\mathbf{w}_{\perp}^{\top}\mathbf{r}.
\]
Since $\mathbb{E}^{s}[\mathbf{u}]=\mathbf{0}$, $\mathbb{E}^{s}[\mathbf{r}]=\mathbf{0}$,
$\mathbb{E}^{t}[\mathbf{u}]=\boldsymbol{\mu}^{t}$, and
$\mathbb{E}^{t}[\mathbf{r}]=\boldsymbol{\mu}_{\perp}^{t}$,
\[
\Delta_{\mu}=\mathbb{E}^{t}[\hat{y}]-\mathbb{E}^{s}[\hat{y}]
=\mathbf{a}^{\top}\boldsymbol{\mu}^{t}+\mathbf{w}_{\perp}^{\top}\boldsymbol{\mu}_{\perp}^{t}.
\]

Define
\[
A:=\sqrt{\mathbf{a}^{\top}\mathbf{\Lambda}^{s}\mathbf{a}},
\quad
x:=\sqrt{(\boldsymbol{\mu}^{t})^{\top}(\mathbf{\Lambda}^{s})^{-1}\boldsymbol{\mu}^{t}},
\quad
B:=\sqrt{\tau}\,\|\mathbf{w}_{\perp}\|_{2},
\quad
y:=\sqrt{\frac{\|\boldsymbol{\mu}_{\perp}^{t}\|_{2}^{2}}{\tau}}.
\]
By Cauchy--Schwarz,
\[
\left|\mathbf{a}^{\top}\boldsymbol{\mu}^{t}\right|
=\left|((\mathbf{\Lambda}^{s})^{1/2}\mathbf{a})^{\top}((\mathbf{\Lambda}^{s})^{-1/2}\boldsymbol{\mu}^{t})\right|
\le Ax,
\]
and
\[
\left|\mathbf{w}_{\perp}^{\top}\boldsymbol{\mu}_{\perp}^{t}\right|
\le
\|\mathbf{w}_{\perp}\|_{2}\,\|\boldsymbol{\mu}_{\perp}^{t}\|_{2}
=By.
\]
Hence
\[
|\Delta_{\mu}|
\le Ax+By
\le \sqrt{A^{2}+B^{2}}\,\sqrt{x^{2}+y^{2}}.
\]
It remains to prove $x^{2}+y^{2}\le 2\mathcal{D}_{\mathrm{PSC}}$.

For the support block,
\[
\operatorname{SKL}\bigl(\mathcal{N}(\mathbf{0},\mathbf{\Lambda}^{s}),\mathcal{N}(\boldsymbol{\mu}^{t},\mathbf{\Sigma}^{t})\bigr)
=
\frac{1}{2}\Bigl[(\boldsymbol{\mu}^{t})^{\top}\bigl((\mathbf{\Lambda}^{s})^{-1}+(\mathbf{\Sigma}^{t})^{-1}\bigr)\boldsymbol{\mu}^{t}
+
\operatorname{tr}\!\bigl(\mathbf{\Lambda}^{s}(\mathbf{\Sigma}^{t})^{-1}+\mathbf{\Sigma}^{t}(\mathbf{\Lambda}^{s})^{-1}-2\mathbf{I}_{K}\bigr)\Bigr].
\]
Since $(\mathbf{\Sigma}^{t})^{-1}\succ\mathbf{0}$ and
$\operatorname{tr}(\mathbf{M}+\mathbf{M}^{-1}-2\mathbf{I}_{K})\ge 0$ for any $\mathbf{M}\succ\mathbf{0}$,
\[
\operatorname{SKL}\bigl(\mathcal{N}(\mathbf{0},\mathbf{\Lambda}^{s}),\mathcal{N}(\boldsymbol{\mu}^{t},\mathbf{\Sigma}^{t})\bigr)
\ge
\frac{1}{2}(\boldsymbol{\mu}^{t})^{\top}(\mathbf{\Lambda}^{s})^{-1}\boldsymbol{\mu}^{t}
=\frac{1}{2}x^{2}.
\]

For the residual block, let $\mathbf{U}_{\perp}\in\mathbb{R}^{D\times(D-K)}$ satisfy
\[
\mathbf{U}_{\perp}^{\top}\mathbf{U}_{\perp}=\mathbf{I}_{D-K},
\qquad
\mathbf{U}_{\perp}\mathbf{U}_{\perp}^{\top}=\mathbf{P}_{\perp}^{s}.
\]
Define
$\tilde{\mathbf{r}}^{s}:=\mathbf{U}_{\perp}^{\top}\mathbf{r}^{s}$,
$\tilde{\mathbf{r}}^{t}:=\mathbf{U}_{\perp}^{\top}\mathbf{r}^{t}$, and
$\tilde{\boldsymbol{\mu}}_{\perp}^{t}:=\mathbf{U}_{\perp}^{\top}\boldsymbol{\mu}_{\perp}^{t}$.
Then
\[
\tilde{\mathbf{r}}^{s}\sim\mathcal{N}(\mathbf{0},\tau\mathbf{I}_{D-K}),
\qquad
\tilde{\mathbf{r}}^{t}\sim\mathcal{N}(\tilde{\boldsymbol{\mu}}_{\perp}^{t},\nu^{t}\mathbf{I}_{D-K}),
\]
and
\[
\operatorname{SKL}\bigl(\mathcal{N}(\mathbf{0},\tau\mathbf{P}_{\perp}^{s}),\mathcal{N}(\boldsymbol{\mu}_{\perp}^{t},\nu^{t}\mathbf{P}_{\perp}^{s})\bigr)
=
\frac{1}{2}\Bigl[\left(\frac{1}{\tau}+\frac{1}{\nu^{t}}\right)\|\tilde{\boldsymbol{\mu}}_{\perp}^{t}\|_{2}^{2}
+
(D-K)\left(\frac{\tau}{\nu^{t}}+\frac{\nu^{t}}{\tau}-2\right)\Bigr].
\]
Because $\nu^{t}>0$ and
$\frac{\tau}{\nu^{t}}+\frac{\nu^{t}}{\tau}-2=\frac{(\tau-\nu^{t})^{2}}{\tau\nu^{t}}\ge 0$,
\[
\operatorname{SKL}\bigl(\mathcal{N}(\mathbf{0},\tau\mathbf{P}_{\perp}^{s}),\mathcal{N}(\boldsymbol{\mu}_{\perp}^{t},\nu^{t}\mathbf{P}_{\perp}^{s})\bigr)
\ge
\frac{1}{2}\,\frac{\|\tilde{\boldsymbol{\mu}}_{\perp}^{t}\|_{2}^{2}}{\tau}
=\frac{1}{2}\,\frac{\|\boldsymbol{\mu}_{\perp}^{t}\|_{2}^{2}}{\tau}
=\frac{1}{2}y^{2}.
\]

Summing the two bounds yields
\[
\mathcal{D}_{\mathrm{PSC}}\ge\frac{1}{2}(x^{2}+y^{2}),
\qquad
x^{2}+y^{2}\le 2\mathcal{D}_{\mathrm{PSC}}.
\]
Therefore,
\[
|\Delta_{\mu}|
\le
\sqrt{2\,\mathcal{D}_{\mathrm{PSC}}}
\sqrt{\mathbf{a}^{\top}\mathbf{\Lambda}^{s}\mathbf{a}+\tau\|\mathbf{w}_{\perp}\|_{2}^{2}}.
\]
Finally, $\sqrt{A^{2}+B^{2}}\le A+B$ gives the looser bound.
\end{proof}

\end{document}